\documentclass[11pt]{extarticle}

\usepackage[english]{babel}
\usepackage[utf8x]{inputenc}
\usepackage[T1]{fontenc}
\usepackage[]{comment}

\usepackage[margin=1in]{geometry}

\usepackage{amsmath}
\usepackage[fleqn]{nccmath}
\usepackage{graphicx}
\usepackage[colorinlistoftodos]{todonotes}
\usepackage[colorlinks=true, allcolors=blue]{hyperref}
\usepackage{amsthm}
\newtheorem{theorem}{Theorem}[section]
\newtheorem{corollary}{Corollary}[theorem]
\newtheorem{lemma}[theorem]{Lemma}
\newtheorem{definition}{Definition}
\newtheorem{example}{Example}
\DeclareMathOperator{\Ber}{Ber}

\title{Online Multi-Armed Bandit}
\author{Uma Roy \thanks{Department of Mathematics, MIT. \texttt{umaroy@mit.edu}} \and Ashwath Thirumalai \thanks{Department of Mathematics, MIT. \texttt{ashwath@mit.edu}} \and Joe Zurier \thanks{Department of Mathematics, MIT. \texttt{jayzee@mit.edu}}}

\begin{document}
\clearpage \maketitle
\thispagestyle{empty}

\begin{abstract}
We introduce a novel variant of the multi-armed bandit problem, in which bandits are streamed one at a time to the player, and at each point, the player can either choose to pull the current bandit or move on to the next bandit. Once a player has moved on from a bandit, they may never visit it again, which is a crucial difference between our problem and classic multi-armed bandit problems. In this online context, we study Bernoulli bandits (bandits with payout $\Ber(p_i)$ for some underlying mean $p_i$) with underlying means drawn i.i.d. from various distributions, including the uniform distribution, and in general, all distributions that have a CDF satisfying certain differentiability conditions near zero. In all cases, we suggest several strategies and investigate their expected performance. Furthermore, we bound the performance of any optimal strategy and show that the strategies we have suggested are indeed optimal up to a constant factor. We also investigate the case where the distribution from which the underlying means are drawn is not known ahead of time. We again, are able to suggest algorithms that are optimal up to a constant factor for this case, given certain mild conditions on the universe of distributions.
\end{abstract}
%include that we can cover MANY distributions by sandwich theorem 'analytic around 0' comes to mind as a phrase that should be in there and anything with bounded density at 0

\newpage 
\setcounter{page}{1}

\section{Introduction}

Multi-armed bandit problems have been studied extensively in the literature. The classical version of the problem, often regarded as the canonical example of the exploration vs. exploitation tradeoff, is formulated as follows. There are $N$ ``bandits", an unknown set of distributions $\{D_i : 1\leq i \leq n\}$, and a maximum number of allowed pulls, $K$. Each bandit can be pulled any number of times, and when bandit $i$ is pulled, it provides a payout $p \simeq D_i$. The problem is formulated as sequentially deciding which of the $N$ bandits to pull at each of the $K$ stages to maximize payout. For the rest of this paper, we fix $N$ and $K$ to mean exactly the number of bandits and the number of pulls respectively. The bandit problem has many variants that have been explored previously in the literature. A survey detailing many of the variants can be found at \cite{MAB_survey}. We describe a few particularly common variants here. 

\begin{itemize}
    \item The \it{Bernoulli bandit problem} assumes the distributions $D_i$ are Bernoulli distributions with means $p_i$, where $p_i$ is independently drawn uniformly at random from $[0,1]$ for each $i$. 
    \item The \it{Pure Exploration Setting} focuses on finding the bandit with the highest expected payout, instead of maximizing the total payout over the course of the exploration and exploitation stages \cite{pure_exploration}.
    \item The \it{contextual bandit problem} includes a context vector that is provided to the player at each stage---the player uses the context vector to inform his choice of bandit. 
\end{itemize}

The most significant application of the classical multi-armed bandit problem is online advertising. Online advertising algorithms decide whether to show a new advertisement or one for which the user's history indicates a high likelihood of clicking. In doing so, these algorithms are solving a multi-armed bandit problem and making the tradeoff between exploration and exploitation.

\subsection{Our Problem}
In this paper, we introduce and investigate the following online variant of the multi-armed bandit problem. This variant, to the best of our knowledge, is a \emph{novel} variant of the classical multi-armed bandit problem, and all of our results in this area are original. Consider $N$ Bernoulli bandits whose underlying means are drawn i.i.d. from a  distribution $F$. Now add the restriction that the bandits are streamed to the player one at a time. At each stage, the player decides whether to pull the lever on the current bandit. If the player decides to pull the lever, the stream is not advanced, and the player receives a payout from the current bandit drawn from its underlying distribution. If the player decides to not pull the lever, the stream is advanced and the player is presented with the next bandit. Once the stream reaches the last bandit, it stays there regardless of the player's actions, until the player has used up all $K$ pulls. A key property of this formulation is that once the player decides not to pull a given bandit, he or she can never return to the same bandit and pull it again---making this problem fundamentally different from most classical formulations of the multi-armed bandit.

\subsubsection{Motivation}
Inspiration for this formulation of the multi-armed bandit problem is derived from the famous secretary problem. Recall that the secretary problem consists of an agent determining the payoff of a stream of potential secretaries, and the agent must, at some point during the stream of secretaries, decide to hire a given secretary, with the goal of maximizing payoff. Similar to the secretary problem, this formulation of the multi-armed bandit problem involves deciding whether to continue with a given option (pull the lever on the current bandit again) or proceed to the next bandit, with no possibility of ever returning. %A crucial difference between our problem and the classical secretary problem is that the underlying means of our bandits are not chosen adversarially---instead they're drawn i.i.d. from some distribution, which allows us to give tight bounds for various strategies we propose. Other crucial differences between the 2 problems include that we are penalized for the length of time we spend pulling sub-optimal bandits, and we do not only care about the rank of the bandits, as we would in the seecretary problem---we care about the differences in their expected payouts.

\subsubsection{An Application to Crowdsourcing}
Similar to how the traditional multi-armed bandit can be applied to online advertising, we propose an application of our variant of multi-armed bandit to crowd-sourcing. Crowd-sourcing platforms such as Amazon's Mechanical Turk allow you to assign tasks to their pool of workers. Although there are many configurable settings with Mechanical Turk, for simplicity, we can model the process of completing tasks through Mechanical Turk or other crowd-sourcing platforms as assigning a series of similar tasks (such as answering survey questions, or labeling images) to a pool of workers, who are presented to the agent one at a time. A current worker's performance on a given task can be modeled as a bandit, where payout corresponds to the worker completing the task correctly (payout $1$) or not (payout $0$). In this analogy, the workers' performance is a Bernoulli bandit, with mean $p$, which represents the worker's average capability for completing the task correctly. Given a current worker, we either can make the choice to assign them more of the tasks, or if their performance is insufficient, we can request a new worker. However, once we request a new worker, we cannot go back to the previous worker (as they likely have been assigned to another user and are no longer available to us). With this analogy, there is an obvious parallel between our problem as posed with streaming Bernoulli bandits and the problem of crowd-sourcing.

\subsection{Notation and Definitions}
Given sufficient motivation for our problem, we go into further technical details and definitions that will serve useful throughout our paper. Throughout this paper, we fix the following notation: $N$ is the number of bandits we are considering and $K$ is the number of pulls available. In all cases, our bandits have some underlying mean $p_i$ for $1 \leq i \leq N$ and give payouts according to $\Ber(p_i)$. The $p_i$ are drawn i.i.d. from some distribution $F$ with bounded support. Note that we sometimes study ``fixed payout bandits" (where the bandit pays out $p_i$ consistently) to obtain bounds on the optimality of strategies. There are 2 flavors of problems we study:

%JOE FIX UP
\begin{itemize}
% fixed the support, support has to be in that interval because otherwise
% bernoulli does not make sense
    \item (Known distribution case): We know the distribution $F$ from which the underlying means of the bandits are drawn. Our results apply for all distributions $F$ with support in $[0,1]$ that have a PDF satisfying some regularity conditions in a neighborhood of $0$.
    \item (Unknown distribution case): We do not know the distribution $F$ from which the underlying means of the bandits are drawn. We examine distributions $F$ drawn from a set of equicontinuous distributions with densities at $0$ bounded above and below by positive constants.
\end{itemize}

We refer interchangeably to bandits as coins and pulls of a bandit as giving heads or tails ($1, 0$ resp.), as each Bernoulli bandit with mean $p$ can be thought of a biased coin with probability $p$ of giving a heads (payout $1$). In the classic multi-armed bandit problem, the goal is to maximize payout. In this paper, we solve (the equivalent) but reverse problem---we wish to minimize payout. This scenario can be thought of as bandits outputting ``harm" or ``loss", and for a given number of pulls, the goal now becomes to minimize ``total loss". In this regime, getting a $0$ (tails) from a bandit is favorable, and getting a $1$ (heads) from a bandit is unfavorable. We solve this version of the problem to simplify some of the algebra involved, although it is equivalent to the problem of maximizing payout. Our goal is to minimize the total loss of a strategy, which is equivalent to minimizing the expected loss per flip (equivalently, expected loss per pull) for a strategy. We formally define both below:

\begin{definition}[Loss of a strategy] \label{loss}
For $K$ pulls and $N$ bandits, let the expected number of heads (or equivalently losses) of some strategy $S$ be $S(N,K)$. This is the \emph{total loss} of a strategy. Then the \emph{expected loss per flip} (or equivalently \emph{expected loss per pull}) of our strategy is $\frac{S(N,K)}{K}$. Intuitively the quantity $\frac{S(N,K)}{K}$ measures the expected value of how much loss we expect per pull.
\end{definition}

We illustrate this terminology in an example below. 

\begin{example}
On an individual pull, a Bernoulli bandit either delivers $0$ loss or $1$ loss (or a fixed payout bandit might deliver $0.6$ loss for example). If a strategy pulls bandits $3$ times such that each time a loss of $1$ is delivered, the total loss of the strategy is $3$ and the loss per flip is $1$, since there were $3$ flips.
\end{example}

We are also interested in the following quantity, which measures the sub-optimality of our strategy vs. an all-knowing oracle. 
\begin{definition}[Sub-optimality of a strategy]
For $K$ pulls and $N$ bandits, let $m$ denote the expected minimum of $N$ draws from the distribution $F$, from which our bandit means are chosen. An all-knowing oracle would have $m$ loss per flip in expectation, and this is the best we can hope for. If a strategy $S$ has $s$ expected loss per flip, the \emph{sub-optimality} of strategy $S$ is measured as $s-m$.
\end{definition}

%One might note that if there are $N$ bandits with means distributed uniformly on $[0,1]$, then the expected value for the maximum mean of the bandits is $\frac{N}{N+1}$, so the maximum value of success per pull is at most $\frac{N}{N+1}$ if we assume that there is some all-knowing oracle that selects the bandit with the highest mean and only pulls it. Thus one could argue it is more intuitive to set the quantity $\frac{N}{N+1} - \frac{S(N,K)}{K}$ as our ``loss", but we choose not to do this in order to simplify our analysis and note that if we're interested in this metric of loss, we can simply subtract the relevant quantity from our definition of the loss. 

%Many of the methods used in this paper are derived from solutions to either the multi-armed bandit problem, which we outlined in the previous section, or the secretary problem. While we analyze them for binary bandits, we are optimistic that similar methods/analyses should carry over for more general bandit distributions. Throughout this paper, we fix the following notation: $N$ is the number of bandits (or coins) we are considering and $K$ is the number of pulls (or flips) available. 

%We see right away that this problem has a number of properties that make it nice to analyze. Perhaps the most important is this: The entire current state can be specified knowing only the number of total lever pulls remaining, the number of bandits remaining, and the history of the current bandit. This makes it possible to write down recursive expressions for the optimal expected value. 

We start by bounding the optimal performance of strategies in the known distribution case. The first distribution that we analyze is $F = \it{U}[0,1]$ --- the simplest of distributions. The insights gleaned from studying this case easily generalize to the other cases we study, but the proofs in this case are more intuitive and thus are presented first.

\section{Known Distribution: Uniform Case}

Throughout this section, we analyze the case where the means of the bandits are drawn uniformly from $[0,1]$. We derive lower bounds on the loss of any strategy by examining fixed payout bandits, and then we provide strategies that achieve these lower bounds (up to a constant factor) on the Bernoulli bandits. 

\subsection{Lower bounds on optimal strategies with Fixed Payout Bandits}

Instead of immediately considering Bernoulli bandits, we will first consider $\textbf{fixed payout bandits}$, where each bandit $i$ has some fixed loss $\mu_i$ for $\mu_i \in [0,1]$. Like the Bernoulli bandits, we will assume that these fixed losses are uniformly distributed on $[0,1]$ . Note that for the fixed payout bandit case, once we have pulled a bandit once, we have complete information about its payout distribution, unlike the Bernoulli bandit case, where uncertainty remains. 

We first derive lower bounds in terms of $N$ and $K$ for the expected loss for an arbitrary strategy on fixed payout bandits. Note that this implies a lower bound for expected loss for any strategy applied to Bernoulli bandits, since in the fixed payout case, we have strictly more information. We start by letting $f(K,N)$ be the expected total loss of an optimal strategy for $K$ pulls and $N$ bandits. We investigate the asymptotic behavior of $f(K, N)$, starting by considering the convergence of $\frac{f(K, N)}{K}$, which is the expected loss per pull of an optimal strategy, for fixed $N$ as we increase $K$ (our number of pulls).

\begin{theorem}[$\beta$ numbers]
Let $f(K,N)$ denote the expected total loss of an optimal strategy for $K$ pulls and $N$ bandits. Let $\beta_N = \lim_{K \rightarrow \infty} \frac{f(K, N)}{K}$. Then for sufficiently large $N$, $\beta_N \rightarrow \frac{2}{N}$.
\end{theorem}

\begin{proof} 
We derive a recursion for $\beta_i$ as follows: If the first of the $i$ bandits has a loss less than than $\beta_{i-1}$, then we should settle on the first bandit, as evidently, our loss per flip will only increase if we move onto the next bandits. In this case, the expected loss per flip of staying with the first bandit is given by the following integral, where we integrate over all possible losses $x$ for which we stay with the first bandit, and the integrand contains the expected payout (which is simply $x$):

\[
\int_0^{\beta_{i-1}} x dx =  \frac{1}{2} \beta_{i-1}^{2}.
\]

If we move on to the next bandit, which happens with probability $1-\beta_{i-1}$, then our expected loss per flip, by definition, is $\beta_{i-1}$. Thus we see the following recursion holds true: 
\[
\beta_i = \frac{1}{2} \beta_{i-1}^{2} + (1-\beta_{i-1})\beta_{i-1} = \beta_{i - 1} - \frac{1}{2}\beta_{i - 1}^{2}.  
\]

Rewriting the above as $\beta_i-\beta_{i-1}=- \frac{1}{2}\beta_{i - 1}^{2}$, we can approximate this continuously with a differential equation, namely, $\beta'(i) = -\frac{1}{2}\beta(i)^{2}$, and solve for the asymptotic approximation of $\beta_i$ (since we know $\beta_i \rightarrow \beta(i)$ as $i \rightarrow \infty$). Doing so gives $\beta_i \rightarrow \beta(i) =\frac{2}{N}$, as desired. 

\end{proof} 

The above theorem shows that any strategy for $K$ pulls and $N$ bandits must have asymptotic expected loss per flip (as $K \rightarrow \infty$) at least $\beta_N = \frac{2}{N}$. However, this bound is only valid as $K \rightarrow \infty$. Below we provide a bound that is valid for all $K$, but gives a tighter bound in the case that $K = o(N^2)$. 

\begin{theorem}
\label{point_sqrtK}
Given $K$ pulls and $N$ bandits, no strategy can attain an expected total loss below $\frac{\sqrt{K}}{8}$, or expected loss per flip below $\frac{1}{8\sqrt{K}}$.
\end{theorem}
\begin{proof}
To reach a bandit with expected loss per flip below $\frac{1}{\sqrt{K}}$ (which we obviously must do to achieve the given total loss), with constant probability $\geq\frac{1}{2}$ we must see at least $\frac{\sqrt{K}}{2}$ bandits by the union bound, since each bandit has probability $\frac{1}{\sqrt{K}}$ of having loss below this threshold. Thus with probability at least $\frac{1}{2}$, we must encounter at least $\frac{\sqrt{K}}{2}$ bandits before we reach a bandit with sufficiently low loss. Note that for each of these bandits, their expected loss per flip is at least $\frac{1}{2}$ (which is the expected value of the uniform distribution, from which their means are drawn), and thus our strategy has total loss with probability $\geq \frac{1}{2}$ equal to at least $\frac{\sqrt{K}}{4}$, since we encounter $\frac{\sqrt{K}}{2}$ bandits with loss at least $\frac{1}{2}$. Thus the expected total loss of our strategy is at least  $\frac{\sqrt{K}}{8}$, as desired. 
\end{proof}

We see that when $K = o(N^2)$, this bound is tighter than the $\frac{2}{N}$ bound given by the asymptotic analysis of the $\beta$ sequence. By establishing both of these bounds for $K$ large and small, in the following section we provide and analyze explicit strategies that achieve the optimal asymptotic loss, up to a constant factor. We note that the bounds we've established are in the context of fixed payout bandits, however they trivially also apply to the Bernoulli bandits, as any strategy for Bernoulli bandits provides at least as good results in the fixed payout case. 

\subsection{An asymptotically optimal strategy for all $K$ for Bernoulli Bandits}
\label{OP}

Below we provide a strategy that has total expected loss per flip within a constant factor of $\frac{2}{N}$ for $K \geq N^2$ and $\frac{1}{\sqrt{K}}$ for $K = o(N^2)$. By the analysis involving fixed payout bandits above, we see that this strategy achieves asymptotically optimal lost (up to a constant factor) for all $K$. We start with the case $K \geq N^2$.

\begin{theorem}
\label{opt-strat-og}
Let $K \geq N^2$. Consider the following strategy: Pull bandit $i$ ($1\leq i\leq N$) repeatedly, and keep track of the total number of times we've flipped it and the total number of heads we've received. We keep flipping the bandit if either of the following conditions holds: 1) We have received at most one total head so far, and the number of times we've flipped it is less than or equal to $N-i$; 2) the number of times we've flipped the bandit is greater than $N-i$. Otherwise, move on to the next bandit. This allows us to achieve a total loss bounded above asymptotically by $\frac{6K}{N}$, or expected loss per flip bounded above by $\frac{6}{N}$. This is within a constant factor of the optimal loss per flip of $\frac{2}{N}$.
\end{theorem}
\begin{proof}
The probability that we get $a$ heads and $b$ tails from a bandit is given by $\int_0^1\binom{a+b}{b}p^a(1-p)^b\mathrm{d}p=\frac{1}{a+b+1}$. If we arrive to bandit $i$, we stay with it when we receive at most $0$ or $1$ heads---i.e. the total number of flips, $a+b = N-i$ and $a=0$ or $a=1$. Hence the probability that we stay with bandit $i$, having arrived to it, is given by $\frac{2}{N-(i-1)}$. The probability that we arrive to and stay with bandit $i$ is given by the probability that we stay with bandit $i$ and don't stay with any previous bandits. This probability is expressed by the following product $$b_i=\frac{2}{N-(i-1)}\prod_{j=1}^{i-1}\left(1-\frac{2}{N-(j-1)}\right)=\frac{2(N-i)}{N(N-1)},$$
where the equality comes from the telescoping of the product.

The posterior distribution of the bandit's mean, having observed $a$ heads and $b$ tails, is $\beta(1+a,1+b)$, which has expected value $\frac{a+1}{a+b+2}$. Thus the expected loss per flip, given that we stay with bandit $i$, it given by averaging the (equally likely) cases where we've seen $0$ or $1$ heads, which gives an expected loss per flip of $\frac{\frac{3}{2}}{N-(i-2)} \leq \frac{\frac{3}{2}}{N-(i-1)}$. 

The expected total loss when sticking with bandit $i$ is less than or equal to $\left(\frac{\frac{3}{2}}{N-(i-1)}\right)(N^2-2N)$, where the first factor comes from the expected loss per flip calculated above, and the second factor comes from the fact that we've only received a total of $\leq 2(i-1)+1\leq 2N$ heads by this point, since we received at most $2$ from each preceding bandit and $1$ from the current one. The expected total loss of this strategy is thus bounded above by the following sum:
\[
\sum_{i=1}^N\left(\frac{\frac{3}{2}}{N-(i-1)}\right)\cdot\left(\frac{2(N-i)}{N(N-1)}\right)(K-2N),
\]
which Mathematica is able to sum explicitly to give the desired upper bound on the total loss.
\end{proof}

Note that the constant factor of $6$ that we get can be made optimal (by decreasing it to $2$, which is the bound given by the $\beta$ numbers) for $K=N^{2+\epsilon}$ by increasing the number of trials for bandit $i$ to $N^{1+\epsilon/2}$, and changing the acceptable number of heads to be $<2N^\epsilon$. We won't go into the details here, however. 

Note that for $K < N^2$ we can recover an asymptotically optimal (up to constant factor) strategy by simply pretending there are only $\sqrt{K}$ bandits, which gives us an expected loss per flip of $\frac{6}{\sqrt{K}}$, which is within a constant factor of the lower bound on expected loss per flip given by Theorem \ref{point_sqrtK} as desired.

\section{Known distribution case: general distributions}
\label{known-gen}

In the previous section, we lower bounded the expected loss of any strategy on the streaming bandit problem, where the bandit means were drawn uniformly from $[0,1]$. In this section, we generalize those bounds and strategies to show similar results for when the bandit means are drawn i.i.d. from a distribution with CDF $F$ that satisfies the following condition: For some $m$, $F$ is $m$ times continuously differentiable in a neighborhood of $0$ and $F^{(m)}(0)>0$.

To do this, we first generalize our results from the previous section to the case when $F$ has CDF of the form $x^m$ for any constant $m$. In general, our results apply for large $N$, and we generally assume that $N \rightarrow \infty$ when we make approximations. Our proofs in this section are very similar to the case where $m=1$ (i.e. $F$ is the uniform distribution). For space reasons, we have placed these proofs in the appendix, as they are more technically involved than the uniform distribution case, but are identical in spirit.

\begin{lemma}
\label{expectation_minimum}
The expectation of the minimum of $N$ draws from the distribution with CDF $x^m$ on the interval $[0,1]$ is asymptotic to $\frac{\Gamma(1/m)}{m N^{1/m}}$ for large $N$.
\end{lemma}

This is an easy computation, whose proof is in the appendix. 

Similar to the case where $F$ is the uniform distribution, we first examine fixed payout bandits to provide a lower bound on the loss of any strategy.

\begin{theorem}[A generalization of the $\beta$ numbers]
\label{beta-generalization}
Let $\eta_i$ denote the expected loss per flip with $K \rightarrow \infty$ flips with $i$ fixed payout bandits with means drawn from distribution $F$ with CDF $x^m$. Then $\eta_N \rightarrow \left( \frac{m+1}{mN}\right)^{1/m}$ as $N \rightarrow \infty$. 
\end{theorem}

The proof of this theorem is in the Appendix. We see that the $\eta_N$ provides an asymptotic lower bound on the expected loss per flip for any strategy with $N$ bandits and $K$ flips. However, similar to the uniform case, we again need a theorem that provides a tighter lower bound for when $K$ is small relative to $N$.

\begin{theorem}[A bound on any strategy in terms of $K$]
\label{k-loss-general}
Any optimal strategy with $K$ pulls and $N$ bandits must have total loss at least $\dfrac{m}{4(m+1)}K^{\frac{m}{m+1}}$, or loss per flip at least $\dfrac{m}{4(m+1)K^{\frac{1}{m+1}}}$.
\end{theorem}

The proof of this theorem is once again very similar to the corresponding proof in the uniform case. The full proof is in the Appendix. With these lower bounds in hand, we can finally state a provably optimal strategy for Bernoulli bandits in the case where $F$ has CDF $x^m$.

\begin{theorem}[The optimal strategy]
\label{generalized-strat}
Let $K \geq N^{\frac{m+1}{m}}$ and set $f_i = \left( \frac{m!}{2}(N-(i-1)) \right)^{1/m}$. Consider the following strategy: Pull bandit $i$ ($1\leq i\leq N$) repeatedly, and keep track of the total number of times we've flipped it and the total number of heads we've received. We keep flipping the bandit if either of the following conditions holds: 1) We have received no heads so far, and the number of times we've flipped it is less than or equal to $f_i$ or 2) the number of times we've flipped the bandit is greater than $f_i$. Otherwise, move on to the next bandit. This allows us to achieve an expected loss per flip of $\leq 2^{1/m}eN^{-1/m}$. This is within a constant factor of the optimal loss, by the above analysis. 
\end{theorem}

The proof of this theorem can be found in the Appendix for space reasons, but it is similar in spirit to the proof of Theorem \ref{opt-strat-og}. We note that for $K = o(N^{1+1/m})$, we can employ a similar trick and pretend that there are $K^{\frac{m}{m+1}}$ bandits, which gives us loss $K^{-1/(m+1)}$ up to a constant factor, which is within a constant factor of the optimal loss by Theorem \ref{k-loss-general}.

We now turn to the case of a more general distribution $F$ satisfying a differentiability condition: There exists an $m>0$ for which $F$ is $m$ times continuously differentiable and $F^{(m)}(0)>0$. To deal with such a distribution, it suffices to note that we only care about the left tail of our distribution.  Fix a radius $\delta$ around $0$ in which $F$ is $m$ times continuously differentiable, with derivative equal to $Cm!$ for some $C>0$, and WLOG choose the smallest such $m$ (i.e. the first $m-1$ derivatives at $0$ are all $0$). Then, since $\delta$ is a constant, for large $N$ it is exponentially unlikely that the tail $x>\delta$ affects our analysis of the left tail (which consists of $x=O(N^{-\frac{1}{m}})$), so we are essentially dealing with a CDF that looks like $Cx^m$ near its left tail. (In fact, by continuity of the $m^{th}$ derivative, we can choose $\epsilon>0$ arbitrarily small and bound the left tail of the CDF between $(C+\epsilon)x^m$ and $(C-\epsilon)x^m$ for $N$ depending on $\epsilon$.) This is, of course, just a rescaling of the CDF we dealt with above, $x^m$. The CDF can be rewritten as $(C^\frac{1}{m}x)^m$, and we see that in fact the strategy for this case is identical to the strategy where the CDF is $x^m$, except that we uniformly scale all thresholds and expectations by $C^{-\frac{1}{m}}$ (that is, a larger $C$ allows us to demand lower thresholds and have lower expected loss, since this means we have more density at the left tail). We can intuitively see this by noting that the quantile function in the presence of a constant $C$ is multiplied by $C^{-\frac{1}{m}}$; this means that a bandit in the $x^m$ regime that is just as good by percentile as one in the $Cx^m$ regime will have a larger loss by a factor of $C^\frac{1}{m}$.

\section{Unknown Distribution Case}

Having considered the case where we know the distribution from which the underlying means of the Bernoulli bandits are drawn, we consider the case where we no longer know this distribution. This case is much more difficult because we need to estimate what the CDF of the distribution is as we execute pulls in order to apply the theorems above. Note that this sort of estimation cannot be done for an arbitrary CDF with no conditions whatsoever, to any reasonable degree of optimality. Morally, this holds because we have no control over how quickly the CDFs converge to their behavior around $0$ (what we deem as the left tail of the distribution). Without some kind of equicontinuity condition it's easy to come up with countable sets of CDFs where each one is slower to converge to its left tail than the previous one, such that for any fixed $N$, it would be impossible to obtain a small loss. 

However, under stronger conditions we can decide which CDF we are looking at, given that it is drawn from a pre-specified family of distributions with certain regularity conditions. Specifically, we consider families of distributions with CDFS that satisfy the following 2 conditions:

\begin{itemize}
\item The family contains CDFs that are $m$ times differentiable in some neighborhood of $0$
\item Each of the first $m-1$ derivatives at $0$ are $0$, while the $m^{th}$ derivatives lie in some bounded positive interval $[\frac{1}{B},B]$
\end{itemize}

We will sketch the algorithm and proof of near-optimality below; we assume WLOG that $K>N^\frac{m+1}{m}$ (otherwise we ignore some machines). Note also that we can extend this result to the case where $m$ is contained in some finite set without much difficulty.

The algorithm consists of two stages. First, we spend some flips observing the CDF; then, we convert this into an estimation of the relevant parameter (density at $0$) and proceed as above. We first will describe the machinery involved in the case of the fixed-payout bandits.

\begin{lemma}
Suppose we pull each of the first $N^\frac{9}{10}$ bandits and directly observe the true mean, and construct an estimator of the density as follows: Divide the $N^\frac{9}{10}$ samples into $N^\frac{7}{10}$ pools of $N^\frac{1}{5}$ samples each, and take the minimum from each sample. Then, average the minima together and multiply by $\frac{N^\frac{1}{5m}}{\Gamma(1+\frac{1}{m})}$ to arrive at an estimate of $C^{-\frac{1}{m}}$, where $C=F^{(m)}(0)$. This estimate is accurate to $O(N^{-\epsilon})$ with high probability, for some $\epsilon>0$.
\end{lemma}
\begin{proof}
The distribution of the minimum of $L=N^\frac{1}{5}$ samples from $F$ is given by $F^{-1}$ composed with Beta$(1,L)$, the first order statistic of the uniform distribution. Approximating $F$ with $Cx^m$ near its left tail (which is valid for $N$ sufficiently large; we can choose such an $N$ using equicontinuity of the $m^{th}$ derivative at $0$) gives us a random variable $\left(\frac{Y}{C}\right)^\frac{1}{m}$, where $Y\sim$Beta$(1,L)$. This distribution has asymptotic mean $\mu = (LC)^{-\frac{1}{m}}\Gamma(1+\frac{1}{m})$ for $N$ large. Taking the average of $N^\frac{7}{10}$ such distributions and applying Hoeffding's inequality yields a probability that the deviation of the average of the means from the expectation $(N^\frac{1}{5}C)^{-\frac{1}{m}}\Gamma(1+\frac{1}{m})$ by more than the smaller-order $N^{-\frac{1}{4m}}$ as less than or equal to $2\exp\left(2N^\frac{7}{10}N^{-\frac{1}{2m}}\right)$, which is exponentially small for $m\geq 1$. Therefore, we have met the goal set out in our lemma: We have shown that the estimate is accurate to within a polynomially small amount, with exponentially high probability.
\end{proof}
At this point, we're done: We have an accurate estimate of $C$ and hence the PDF near $0$, so we can apply the strategy in Theorem 4.4 on the remaining bandits. Since we only used $N^{9/10}$ flips, which is equal to $o\left(\frac{K}{N^\frac{1}{m}}\right)$, this doesn't affect our expected loss asymptotically.

The remaining difficulty is to show that we can tackle the Bernoulli case, which is much more difficult as we now only have a noisy estimate of the true means of the bandits, making it difficult to determine the distribution from which those means were drawn. However, we provide an estimation procedure below that resolves this difficulty. % One would think that this would be very difficult, because there are now errors in trying to observe individual samples from the PDF that we might expect to compound the errors from combining the samples into an estimate of density at zero. Also, it would seem that we need to waste precious flips getting accurate PDF estimates. However, it turns out that there is a nice way to avoid both problems at once by 

\begin{theorem}[Unknown Distribution for Bernoulli Bandits]
Consider the following sampling procedure: Flip a given bandit $i$ until we receive a heads, then let $\mu_i$ denote the reciprocal of the number of tails we received. Then, the quantity $(m!)^\frac{1}{m}\mu_i$ is asymptotically (towards the left of the distribution) distributed the same as the underlying distribution $F$. 
\end{theorem}

\begin{proof}
Suppose for the moment that our underlying CDF for the distribution of the means of the bandits is $F(x)=x^m$ (which it is, up to a constant scaling factor and ignoring the right tail) and consider the probability of getting exactly $T$ tails followed by a single head ($T\geq 0$). For a given underlying mean $p$, this will be $(1-p)^Tp$, so integrate that against the PDF to get
\[
\int_0^1(1-p)^Tpmp^{m-1}dp=m\beta (1+T,1+m)\sim\frac{m\cdot m!}{T^{m+1}}.
\]
Thus, $\mu_i$ is distributed according to the discrete distribution with mass $\frac{m\cdot m!}{T^{m+1}}$ on each point $\frac{1}{T}$. For $T$ large, this is very well approximated by the continuous distribution $x^m$ near $0$, since 
\[
\int_\frac{1}{T}^\frac{1}{T-1}mx^{m-1}dx=\frac{T^m-(T-1)^m}{T^m(T-1)^m}\sim\frac{m}{T^{m+1}}
\]
which, up to a factor $m!$, is asymptotically just the mass on that point. In CDF space, the extra factor of $m!$ translates to a multiplicative factor of $\kappa x^m$ where $\kappa=(m!)^\frac{1}{m}\sim\frac{m}{e}$; the distributions are asymptotically identical up to this scaling factor, and therefore we can sample from the PMF (by taking the reciprocal of the string of heads) and multiply by $(m!)^\frac{1}{m}$ to arrive at a representative sample from the true PDF of the distribution of the means. Note that the above analysis still holds when we pass in a constant factor $C$, allowing us to do the estimation required to apply the lemma above.
\end{proof}
\begin{corollary}
We can apply Lemma 4.1 to the Bernoulli case, thereby solving the given unknown distribution scenario.
\end{corollary}
\begin{proof}
Simply use the $\mu_i$ as our samples of the PDF of the distribution of the means, and note that because we only sample from $N^\frac{9}{10}$ machines, and we waste at most one flip per machine, we in total lose at most $N^\frac{9}{10}=o\left(\frac{K}{N^\frac{1}{m}}\right)$.
\end{proof}

Thus for Bernoulli bandits, our algorithm allows us to achieve optimal expected loss (up to a constant factor) when we do not know the true distribution of the means of the bandits, and only know certain conditions on the family from which this distribution is taken. 

To deal with a finite set of different possible $m$, it suffices to increase $N$ large enough so that we can tell when our estimates of the density at zero are outside the interval $[\frac{1}{B_m},B_m]$. That is, if our density is $\frac{1}{10B_m}$, then with (exponentially) high probability this is too small and the true density is $0$ for that $m$. Similarly, our estimates of the density for the $m^{th}$ derivative at zero will blow up as a function of $N$ if the $m^{th}$ derivative is not the lowest order nonzero derivative at 0.
% remember to justify multiple m! this requires a couple of tricks
% how big does N need to be, by continuity? Discuss
\section{Conclusion and Future Work}
%  and asymptotic variance $\sigma^2 = (LC)^{-\frac{2}{m}}\left[\Gamma(1+\frac{2}{m})-\Gamma(1+\frac{1}{m})^2\right]$
In this paper, we consider a novel (to our knowledge) variation of the multi-armed bandit problem that combines the idea of the fundamental exploration-exploitation tradeoff present in bandit problems with the idea of streaming input from the classic secretary problem. For cases where the bandit means are drawn from some (known) distribution $F$ satisfying some mild differentiability conditions, we prove lower bounds on the loss of any algorithm and describe an algorithm that achieves the lower bound of loss up to a constant factor. For cases where the bandit means are drawn from some (unknown) distribution $F$ taken from an equicontinuous universe of PDFs with similar behavior near $0$, we manage to do the same. Our model has potential applications to crowd-sourcing, as described by the analogy in the introduction.

\subsection{Future Work}
There are several interesting future directions in which to take this work. In the unknown distribution case, it would be nice to develop a better algorithm for learning the (left tail of the) CDF on the fly, perhaps with fewer constraints, or to show that no such algorithm exists. It also seems interesting to consider the distributions not tackled in this paper; for instance, smooth PDFs supported on $[0,1]$ that vanish to all orders at $0$. Finally, this problem makes sense for Bandits which have more complicated payout schemes than Bernoulli payouts. The authors expect that for distributions with finite first moments, there should be a reasonable way to extend the work in this paper to approach the online bandit problem in this setting (perhaps if the distributions come from an exponential family with parameter varying according to some known or unknown distribution).% better online learning
% better than bernoulli (thoughts)
% general distributions
\section{Acknowledgements} 
We would like to thank Professor David Karger and Professor Aleksander Madry for helpful advice in editing this paper. We would like to thank Professor Robert Kleinberg for useful comments regarding our results. We would like to acknowledge his PhD student Johan Björck for suggesting the application of this model to the crowd-sourcing setting. 

%\bibliographystyle{alpha}
%\bibliography{sample}

\appendix

\section{Proofs of statements in Section \ref{known-gen}}

\subsection{Proof of Lemma \ref{expectation_minimum}}
\begin{lemma}
The expectation of the minimum of $N$ draws from the distribution with CDF $x^m$ on the interval $[0,1]$ is asymptotic to $\frac{\Gamma(1/m)}{m N^{1/m}}$ for large $N$.
\end{lemma}
\begin{proof}
We see that the CDF of the minimum of $N$ draws from a distribution with CDF $x^m$ is $1-(1-x^m)^N)$, since the probability that all of the draws are $\geq x$ is $(1-x^m)^N$, which is precisely the probability that minimum is $\geq x$. The PDF is readily given by $p(x) = mNx^{m-1} (1-x^m)^{N-1}$. Given the PDF, it is trivial to calculate the expectation of the minimum, which is given by the expression below. 

\[
\int_0^1 mNx^m(1-x^m)^{N-1}  dx =  \frac{N \Gamma \left(\frac{1}{m}\right) \Gamma (N)}{m \Gamma \left(N+\frac{1}{m}+1\right)} = \frac{\Gamma \left(\frac{1}{m}\right)}{mN^{1/m}}
\],

using the asymptotic result that $\frac{\Gamma(x+\alpha)}{\Gamma(x+\beta)} = x^{\beta-\alpha}$ for large $x$, which applies as $N \rightarrow \infty$.
\end{proof}

\subsection{Proof of Theorem \ref{beta-generalization}} 
\begin{theorem}[A generalization of the $\beta$ numbers]
Let $\eta_i$ denote the expected loss per flip with $K \rightarrow \infty$ flips, if you have $i$ fixed payout bandits with means drawn from distribution $F$ with CDF $x^m$. Then $\eta_N \rightarrow \left( \frac{m+1}{mN}\right)^{1/m}$ as $N \rightarrow \infty$. 
\end{theorem}

\begin{proof} 
We derive a recursion for $\eta_i$ as follows: If the first of the $i$ bandits has a loss less than than $\eta_{i-1}$, then we should settle on the first bandit, as evidently, our loss per flip will only increase if we move onto the next bandits. In this case, the expected loss per flip of staying with the first bandit is given by the following integral---note that the integrand has the probability that the bandit has the given mean $mx^{m-1}$ (which is the PDF of $F$) times the loss at that mean, which is simply $x$, and the bounds of integration are precisely the values for which we'd stay with bandit $i$.

\[
\int_0^{\eta_{i-1}} x \cdot mx^{m-1} dx =  \frac{m}{m+1} \eta_{i-1}^{m+1}.
\]

% the following ratio of integrals. The integral in the numerator is the expected payout of a bandit given that we stay with it (i.e. given that it has mean between $[0, \eta_i]$). The integral in the denominator is the probability of staying with the bandit. By Bayes' theorem, we get the payout of the bandit,  , where we note that the integrand contains the probability we stay with bandit ($mx^{m-1}$, which is simply the PDF of $F$), times the payout at 
% \[
% \frac{\int_0^{\gamma_{i-1}} x \cdot m \cdot x^{m-1} dx}{\int_0^{\gamma_{i - 1}} m \cdot x^{m - 1} dx} =  \frac{m \gamma _{i-1}}{m+1}.
% \]

If we move on to the next bandit, which happens with probability $1-\gamma_{i-1}^m$, then our expected loss per flip, by definition is $\gamma_{i-1}$. Thus we see the following recursion holds true: 
\[
\eta_i = \frac{m}{m+1} \eta_{i-1}^{m+1} + (1-\eta_{i-1}^m)\eta_{i-1} = \eta_{i - 1} - \frac{1}{m + 1}\eta_{i - 1}^{m + 1}.  
\]

We can turn this into a differential equation, namely, $\eta'(i) = -\frac{1}{m + 1}\eta(i)^{m + 1}$, and solve for the asymptotic approximation of $\eta_i$ (since we know $\eta_i \rightarrow \eta(i)$ as $i \rightarrow \infty$). Doing so gives $\eta_i \rightarrow \eta(i) = \left( \frac{m+1}{mN}\right)^{1/m}$, as desired. 

\end{proof}

\subsection{Proof of Theorem \ref{k-loss-general}}

\begin{theorem}[A bound on any strategy in terms of $K$]
Any optimal strategy with $K$ pulls and $N$ bandits must have total loss at least $\dfrac{m}{4(m+1)}K^{\frac{m}{m+1}}$, or loss per flip at least $\dfrac{m}{4(m+1)K^{\frac{1}{m+1}}}$.
\end{theorem}

\begin{proof}
To achieve a total loss less than the given quantity, we must reach a bandit that has loss per flip $\leq K^{\frac{-1}{m+1}}$. Since each bandit has probability $(K^{\frac{-m}{m+1}})$ of having loss less than this threshold, by union bound, the probability that such a bandit exists in the first $\dfrac{(K^{\frac{m}{m+1}})}{2}$ bandits is $\leq \frac{1}{2}$. Thus with probability $\geq \frac{1}{2}$, we must encounter at least $\dfrac{(K^{\frac{m}{m+1}})}{2}$ bandits before we reach a bandit with a sufficiently low loss. Note that for each of the $\dfrac{(K^{\frac{m}{m+1}})}{2}$ bandits their expected loss per flip is at least $\frac{m}{m+1}$ (the mean of a distribution with CDF $x^m$). Thus in expectation, our total loss is at least $\frac{1}{4}\frac{m}{m+1} (K^{\frac{m}{m+1}})$, for with probability $\geq \frac{1}{2}$, we encounter at least $ \dfrac{(K^{\frac{m}{m+1}})}{2}$ bandits with $\frac{m}{m+1}$ loss each, proving the claim. 
\end{proof}

\subsection{Proof of Theorem \ref{generalized-strat}}
\begin{theorem}[The optimal strategy]
Let $K \geq N^{1+1/m}$ and set $f_i = \left( \frac{m!}{2}(N-(i-1)) \right)^{1/m}$. Consider the following strategy: Pull bandit $i$ ($1\leq i\leq N$) repeatedly, and keep track of the total number of times we've flipped it and the total number of heads we've received. We keep flipping the bandit if either of the following conditions holds: 1) We have received no heads so far, and the number of times we've flipped it is less than or equal to $f_i$ or 2) the number of times we've flipped the bandit is greater than $f_i$. Otherwise, move on to the next bandit. This allows us to achieve an expected loss per flip of $\leq 2^{1/m}eN^{-1/m}$. This is within a constant factor of the optimal loss, by the above analysis. 
\end{theorem}

\begin{proof}
The probability that a bandit with mean drawn from distribution with CDF $x^m$ produces $a$ heads and $b$ tails is given by the following quantity:

\begin{align*}
\int_0^1 mp^{m-1} \binom{a+b}{a} p^a (1-p)^b &= m \binom{a+b}{b} \frac{\Gamma(b+1) \Gamma(a+m)}{\Gamma(a+b+m+1)} \\
&= \frac{m\Gamma(a+b+1)\Gamma(a+m)}{\Gamma(a+1)\Gamma(a+b+m+1)}
\end{align*}

Using the above quantity, we see that if we arrive to bandit $i$, the probability that we stay on it is given by $\dfrac{m\Gamma(f_i+1)\Gamma(m)}{\Gamma(1)\Gamma(f_i+m+1)} = m! \dfrac{\Gamma(f_i+1)}{\Gamma(f_i+m+1)}$ (we simply substitute in $a=0$ heads and $a+b = f_i$ in above expression). Since $f_i \rightarrow \infty$ as $N \rightarrow \infty$ (for all $i < O(N)$), we can use the asymptotic formula $\dfrac{\Gamma(x+\alpha)}{\Gamma(x+\beta)} = x^{\beta-\alpha}$ for large $x$, and get that probability of staying on bandit $i$ is $m! f_i^{-m} = \frac{2}{N-(i-1)}$, when we substitute in the value of $f_i$.

Notice this is the same as the uniform case, thus the probability that you arrive to and stay on bandit $i$ telescopes and it is $\frac{2(N-i)}{N(N-1)}$. 

The posterior PDF $p(x)$ for the bandit's mean given that we've seen $a$ heads and $b$ tails is given by $$
m x^{a+m-1}(1-x)^b\left(m \frac{\Gamma(b+1) \Gamma(a+m)}{\Gamma(a+b+m+1)}\right)^{-1}
$$ with a simple calculation involving Bayes' rule. The expectation of the posterior is given by

\begin{align*}
    \int_0^1 x m x^{a+m-1}(1-x)^b\left(m \frac{\Gamma(b+1) \Gamma(a+m)}{\Gamma(a+b+m+1)}\right)^{-1} dx
    &= \frac{\Gamma(a+m+1) \Gamma(b+1)}{\Gamma(a+b+m+2)}\cdot \left({\frac{\Gamma(b+1) \Gamma(a+m)}{\Gamma(a+b+m+1)}} \right)^{-1} \\
    &= \frac{a+m}{a+b+m+1} 
\end{align*}

If we stay on bandit $i$, we know we have received $0$ heads from $f_i$ flips. Thus the expected payoff, given you've stayed on a bandit $i$, is thus $\dfrac{m}{f_i+m+1}$ (again substituting $a=0$ for the number of heads and $a+b = f_i$ for the number of flips). Substituting in the definition of $f_i$, we get that the expected loss is $\dfrac{m}{C^{1/m}+m+1}$, where $C = \left( \dfrac{m!}{2}(N-(i-1)) \right)$. 

Thus the overall expected total loss for this strategy is upper bounded by  
\begin{align}
&= \sum_{i=1}^{N} \dfrac{2(N-i)}{N(N-1)} \dfrac{m}{C^{1/m}+m+1} (K-N) \\
&\leq \sum_{i=1}^{N} \dfrac{2(N-(i-1))}{N(N-1)} \dfrac{m}{C^{1/m}} (K-N),
\end{align}

since once we stay with a bandit, we've received at most $N$ heads from all previous bandits, so the total number of pulls left and the number of tails we've received thus far sums to at least $K-N$. 

Approximating this sum by an integral, we get (by Mathematica), that the sum is asymptotically (as $N \rightarrow \infty$), 

$$ \frac{2^{1+1/m}m^2}{(2m-1)(m!)^{1/m}} \frac{K-N}{N^{1/m}},$$

which is approximately, using Stirling's approximation for the factorial, 

$$\leq 2^{1/m}e \frac{K-N}{N^{1/m}} = \leq 2^{1/m}e \frac{K}{N^{1/m}},$$

for $K \geq N^{1+1/m}$, giving us an expected loss per flip $\leq 2^{1/m}eN^{-1/m}$. We see by examining the lower bound provided by the $\eta$ numbers that the lower bound on expected loss is $\left(\frac{m+1}{m}\right)^{1/m} N^{-1/m} \geq N^{-1/m}$, so for all $m\geq 1$, we are at most a constant factor of $6$ away from the lower bound on loss of $N^{-1/m}$. 

\end{proof}

\end{document}